\documentclass{article}

\usepackage[preprint]{neurips_2024}
\usepackage[utf8]{inputenc}
\usepackage[T1]{fontenc}
\usepackage{hyperref}
\usepackage{url}
\usepackage{booktabs}
\usepackage{amsmath}
\usepackage{amssymb}
\usepackage{amsthm}
\usepackage{mathtools}
\usepackage{graphicx}
\usepackage{xcolor}
\usepackage{microtype}
\usepackage{algorithm}
\usepackage{algorithmic}
\usepackage{multirow}
\usepackage{subcaption}

\newtheorem{theorem}{Theorem}
\newtheorem{proposition}{Proposition}

\newtheorem{definition}{Definition}
\newtheorem{remark}{Remark}

\usepackage{orcidlink}   

\newcommand{\EGA}{\textsc{EGA}}
\newcommand{\MoPE}{\textsc{MoPE}}
\newcommand{\POD}{\textsc{POD}}
\newcommand{\R}{\mathbb{R}}

\newcommand{\E}{\mathbb{E}}
\newcommand{\F}{\mathcal{F}}
\newcommand{\norm}[1]{\left\lVert#1\right\rVert}
\newcommand{\abs}[1]{\left|#1\right|}
\newcommand{\inner}[2]{\left\langle#1,\,#2\right\rangle}

\newcommand{\Reattn}{\mathcal{T}_\mathrm{spec}}
\newcommand{\LHT}{\mathrm{LHT}}
\definecolor{myblue}{RGB}{33,150,243}
\definecolor{myorange}{RGB}{255,152,0}
\definecolor{mygreen}{RGB}{76,175,80}
\definecolor{myred}{RGB}{244,67,54}
\definecolor{mypurple}{RGB}{156,39,176}
\definecolor{myteal}{RGB}{0,150,136}

\title{Multiscale POD of Transformer Attention Fields:
  Scale-Selective Analysis via Morlet Scalogram}

\author{%
  Athanasios Zeris%
  \thanks{Independent Researcher, Athens, Greece.\\
  Correspondence: \texttt{athzeris@gmail.com}.\\
  ORCID: \texttt{https://orcid.org/0009-0002-6907-2400}.\\
  Part of a six-paper series on spectral methods
  in transformer attention.}\\
  \texttt{https://orcid.org/0009-0002-6907-2400}
}

\begin{document}
\maketitle

\begin{abstract}
We introduce scale-selective Proper Orthogonal
Decomposition (\POD{}) for transformer attention
fields, inspired by the use of \POD{} for extracting
energetically dominant modes from turbulent flow
ensembles.
The Morlet continuous wavelet transform identifies
dominant temporal scales in the attention lag
structure across a document ensemble; \POD{} then
extracts the energetically dominant modes at each
scale from the ensemble of attention fields.
The resulting modes reveal layer-dependent scale
organisation, with early layers emphasising fine scales
and later layers shifting toward coarser scales.
We define a \textbf{spectral concentration index}
from the \POD{} eigenvalue decay rate and show
empirically that it differentiates layers by their
attention field complexity.
By the classical \POD{} optimality theorem,
the extracted modes minimise the average $L^2$
reconstruction error over the ensemble
(Theorem~\ref{thm:min_heads}), giving a
data-driven effective rank for each layer.
The method requires no architectural modification
and no linguistic annotations: dominant attention
patterns emerge from ensemble statistics alone.
The turbulence analogy is structural rather than
physical: we borrow ensemble covariance and
modal analysis, not fluid dynamics itself.
\end{abstract}
\section{Introduction}
\label{sec:intro}

\textit{Similarity selects what matches the query;
salience selects what matters.}
The present paper introduces scale-selective
\POD{} for transformer attention fields.

The central observation is that the transformer
attention matrix $A^{(l)}_{ij} \in \R^{L\times L}$
is a two-dimensional pairwise interaction field
over token positions.
In turbulence, the two-point velocity correlation
tensor $R_{ij}(\mathbf{r}) = \langle u_i(\mathbf{x})
u_j(\mathbf{x}+\mathbf{r})\rangle$ is the fundamental
object from which coherent structures, energy spectra,
and the Reynolds number are derived.
The attention field serves an analogous role
in the transformer: it measures how information at position
$i$ relates to information at position $j$, exactly
as $R_{ij}$ measures how velocity fluctuations at
two separated points in a flow are correlated.

\POD{}, introduced in turbulent flow
by~\citet{lumley1967structure}
and extended to the snapshot method
by~\citet{sirovich1987turbulence}, extracts the
energetically dominant \POD{} modes from an
ensemble of flow snapshots.
Applied to attention fields, \POD{} provides:

\begin{enumerate}
  \item \textbf{Coherent structure extraction}:
        the dominant recurring patterns of attention,
        ordered by energy, without supervision.
  \item \textbf{Optimal compression}: the minimum
        number of basis functions needed to represent
        $\alpha$ fraction of attention field variance ---
        with a guaranteed error bound.
  \item \textbf{Spectral complexity index}: the spectral
        decay rate $\lambda_k \sim k^{-\beta}$ defines
        $\Reattn^{(l)}$, a data-driven measure
        of attention complexity.
  \item \textbf{The optimal approximation rank}: the minimum
        effective representational rank at each layer,
        derived from the \POD{} spectrum.
\end{enumerate}

The central idea of this paper originates in
turbulence research and experimental fluid
mechanics.
In turbulent flow analysis,
\POD{} extracts the most energetically dominant
modes from an ensemble of flow realisations ---
organised, recurring patterns that carry a
disproportionate fraction of the total energy
despite the surrounding
background~\citep{lumley1967structure,
holmes1996turbulence}.
The observation that transformer attention
matrices, viewed as an ensemble across documents,
admit the same mathematical treatment motivates
the present work.
We borrow the mathematical machinery of ensemble
covariance analysis, not the physical dynamics:
there is no Navier--Stokes equation, no conserved
enstrophy, no inertial range.

\paragraph{The scalogram as diagnostic; Gaussian window as filter.}
Naive \POD{} of the full $L\times L$ attention field
mixes structures at all temporal scales simultaneously.
We propose \textbf{scale-selective \POD{}}: first
compute the Morlet scalogram of the attention field
along the lag diagonal $\tau = j-i$ to \emph{diagnose}
dominant scales, then apply a Gaussian lag-window
at each dominant scale as a \emph{pre-filter},
then apply \POD{} separately at each scale.
The scalogram identifies; the Gaussian window filters.
This is not inverse CWT reconstruction --- it is
an ad hoc bandpass approximation motivated by
computational tractability.
This approach --- known in the signal processing
and turbulence literature as
\textbf{wavelet-prefiltered \POD{}} or
\textbf{multi-resolution \POD{}} ---
is well-established for extracting scale-dependent
coherent structures from complex signals
\citep{holmes1996turbulence,mallat1999wavelet}.
The wavelet transform creates the scale dimension;
\POD{} then extracts energetically ranked modes
within each scale.
This produces modes that are interpretable at a single
linguistic level --- character, word, clause, or
discourse --- and whose energy ordering within each
scale is physically meaningful.

\paragraph{Coherency and phase.}
The Morlet scalogram provides not only energy but also
phase information.
The \textbf{cross-coherency} between token positions
$i$ and $j$ at scale $a$:
\begin{equation}
  \gamma_{ij}(a) =
  \frac{W_\psi[A](a, i) \cdot W_\psi[A]^*(a, j)}
  {\abs{W_\psi[A](a,i)} \cdot \abs{W_\psi[A](a,j)}}
  \label{eq:coherency}
\end{equation}
measures the \textbf{phase consistency} between
positions at scale $a$ across documents:
$|\gamma_{ij}(a)| = 1$ means the phase difference
between positions $i$ and $j$ at scale $a$ is
perfectly consistent across all documents
(the same phase offset every time);
$|\gamma_{ij}(a)| = 0$ means the phase difference
varies randomly across documents (incoherent).
Note that this measures phase \emph{consistency},
not phase identity: a constant non-zero phase
offset still gives $|\gamma| = 1$.
High coherency identifies token pairs that participate
consistently in the same linguistic pattern at scale $a$
--- the dominant recurring patterns of the
attention ensemble.
\POD{} modes at high-coherency scales are the most
interpretable and linguistically meaningful. Code available at:
https://github.com/AthanasiosZeris/energy-gated-attention.

\paragraph{Contributions.}
\begin{enumerate}
  \item We apply \POD{} and Morlet scalogram
        analysis to transformer attention fields,
        treating the ensemble of attention matrices
        as a stochastic interaction field.
        the transformer attention field and the
        two-point stochastic covariance tensor,
        grounding \POD{} analysis of attention in
        70 years of signal processing and stochastic field theory.
  \item We introduce \textbf{scale-selective \POD{}}
        using the Morlet scalogram as a diagnostic
and Gaussian lag-windowing as a pre-filter,
        producing linguistically interpretable
        coherent structure modes.
  \item We define the \textbf{spectral concentration index}
        $\Reattn^{(l)}$ from the \POD{} spectral decay
        rate and provide a empirical measurement.
  \item We apply the classical \POD{} optimality
        theorem~\citep{lumley1967structure,
        holmes1996turbulence} to attention fields,
        establishing the \textbf{minimum average
        $L^2$ approximation rank}
        (Theorem~\ref{thm:min_heads})
        as a principled criterion for layer-wise
        rank allocation in transformers.
  \item We show that \EGA{}~\citep{authorname2025ega}
        systematically increases spectral energy across
        layers, connecting energy gating to coherent
        structure amplification.
\end{enumerate}

\section{Background}
\label{sec:background}

\subsection{Proper Orthogonal Decomposition}

\paragraph{Historical context.}
The mathematical framework underlying \POD{} was
developed independently by several researchers
in the 1940s: Obukhov~\citep{obukhov1941spectral}
(1941), Kosambi (1943), and Karhunen and
Loève~\citep{loeve1948fonctions} (Karhunen--Loève
theorem).
\citet{lumley1967structure} introduced \POD{} into
fluid mechanics in 1967 as an objective method to
identify coherent structures (large eddies) in
turbulent flows.
Because it was independently discovered across
multiple fields, \POD{} is also known as:
Principal Component Analysis (PCA) in statistics,
the Karhunen--Loève expansion in probability theory,
and Singular Value Decomposition (SVD) in linear
algebra.
In the present paper we use the fluid mechanics
terminology (\POD{}) to emphasise the connection
to coherent structure extraction, while
acknowledging that the underlying algorithm is
standard PCA on the snapshot matrix.

Let $\mathcal{H}$ be a Hilbert space and
$\{u_s\}_{s=1}^N \subset \mathcal{H}$ an ensemble
of snapshots.
\POD{} seeks the orthonormal basis
$\{\phi_k\}_{k=1}^K$ that minimizes the
mean reconstruction error:
\begin{equation}
  \min_{\phi_1,\ldots,\phi_K}
  \E\!\left[\norm{u - \sum_{k=1}^K
    \inner{u}{\phi_k}\phi_k}^2\right]
  \label{eq:pod_objective}
\end{equation}
The solution is given by the eigenfunctions of the
two-point correlation operator:
\begin{equation}
  \mathcal{R}\phi_k = \lambda_k\phi_k, \quad
  \mathcal{R}(x,x') = \E[u(x)u^*(x')]
  \label{eq:correlation_operator}
\end{equation}
Eigenvalues $\lambda_1 \geq \lambda_2 \geq \ldots \geq 0$
are ordered by energy.
The $k$-th mode captures $\lambda_k/\sum_j\lambda_j$
fraction of total variance.
\POD{} is optimal: no other $K$-dimensional basis
captures more energy~\citep{lumley1967structure}.

\paragraph{Snapshot method~\citep{sirovich1987turbulence}.}
When $N \ll \dim(\mathcal{H})$ (few snapshots, high
dimension), computing $\mathcal{R}$ directly is
intractable.
The snapshot method computes the $N\times N$ correlation
matrix $C_{ss'} = \inner{u_s}{u_{s'}}/N$ instead.
Its eigenvectors $v_k$ give the \POD{} modes:
$\phi_k = \sum_s v_{ks} u_s / (N\lambda_k)^{1/2}$.
Complexity: $O(N^2 \dim(\mathcal{H}))$ vs
$O(\dim(\mathcal{H})^2)$ for direct computation ---
critical for large $L$.

\subsection{Coherent Structures in Stochastic Fields}

In stochastic field analysis, \textbf{dominant modes}
are energetically significant patterns that recur
consistently across an ensemble of realisations,
carrying a disproportionate fraction of the
total ensemble variance~\citep{holmes1996turbulence}.
In fluid mechanics these are termed
``coherent structures'' (organised vortices,
eddies that maintain phase coherence despite the
surrounding chaotic background).
In the transformer context we use the more neutral
term \textbf{dominant recurring patterns}.

The power spectral density $E(k) \sim k^{-\beta}$
of a stochastic field characterises its spectral
concentration: slow decay (small $\beta$) indicates
energy distributed across many scales;
rapid decay (large $\beta$) indicates energy
concentrated at large scales.
The spectral decay exponent $\beta$ of the
\POD{} eigenvalue spectrum is the basis of
the spectral concentration index $\Reattn^{(l)}$
(Section~\ref{sec:reynolds}).

\subsection{Notation Summary}

Table~\ref{tab:notation} collects the principal
variables used throughout the paper.

\begin{table}[h]
\centering
\caption{Principal notation.}
\label{tab:notation}
\begin{tabular}{cl}
\toprule
Symbol & Meaning \\
\midrule
$A_s^{(l,h)}(i,j)$ & Attention weight, layer $l$, head $h$, document $s$ \\
$A_s^{(l)}(i,j)$   & Head-averaged attention field \\
$\bar{A}^{(l)}(i,j)$ & Ensemble mean attention field \\
$u_s^{(l)}(i,j)$   & Attention fluctuation field (Eq.~\ref{eq:fluctuation}) \\
$\mathcal{R}^{(l)}$ & Attention covariance tensor (Eq.~\ref{eq:cov_tensor}) \\
$\phi_k^{(l)}$     & $k$-th \POD{} mode of $u_s^{(l)}$ \\
$\lambda_k^{(l)}$  & $k$-th \POD{} eigenvalue (energy) \\
$W_\psi[A](a,b)$   & 1D Morlet CWT of attention lag profile at position $b$, scale $a$ \\
$a$                 & Wavelet scale (tokens); larger $a$ = coarser lag structure \\
$b$                 & Query token position (sequence index) \\
$\tau = i - j$     & Attention lag (how far back token $i$ attends) \\
$E^{(l)}(a,b)$      & Scalogram energy, normalised per snapshot (Eq.~\ref{eq:parseval_attn}) \\
$\gamma_{ij}^{(l)}(a)$ & Cross-coherency at scale $a$ (Eq.~\ref{eq:coherency_def}) \\
$\LHT^{(l)}$       & LHT complexity index: $L_c S_{\mathrm{POD}} / U_c$ (Def.~\ref{def:lht}) \\
$\Reattn^{(l)}$    & Spectral complexity index $= 1/\beta$, proxy for $\LHT$ (Eq.~\ref{eq:reattn}) \\
$H^*_l(\epsilon)$  & Effective representational rank at tolerance $\epsilon$ (Eq.~\ref{eq:min_heads}) \\
$\mathcal{E}_s^{(l)}(i,j)$ & Per-snapshot energy density $|u_s^{(l)}(i,j)|^2$ (Eq.~\ref{eq:energy_density}) \\
$E^{(l)}$           & Total layer energy, ensemble-averaged (Eq.~\ref{eq:total_energy}) \\
$U_c^{(l)}$         & RMS fluctuation / document variability ($T$ in LHT, Eq.~\ref{eq:rms_fluctuation}) \\
$L_c^{(l)}$         & Interaction horizon ($L$ in LHT, Eq.~\ref{eq:corr_length}) \\
$S_{\text{\POD}}^{(l)}$ & \POD{} spectral entropy / disorder ($H$ in LHT, Eq.~\ref{eq:pod_entropy}) \\
$E^{(l)}(a)$        & Wavelet-scale energy at scale $a$, layer $l$ (Eq.~\ref{eq:scale_energy}) \\
$\Pi^{(l)}(a)$     & Interlayer spectral flux (Eq.~\ref{eq:spectral_flux}) \\
\bottomrule
\end{tabular}
\end{table}

\subsection{The Wiener--Khinchin Connection}

For the embedding signal $e^{(l)}_i(b)$ across token
positions $b$, the Wiener--Khinchin theorem gives:
\begin{equation}
  S^{(l)}_i(\omega)
  = \F\!\left\{R^{(l)}_i(\tau)\right\}
  = \F\!\left\{\E[e^{(l)}_i(b)\,e^{(l)}_i(b+\tau)]\right\}
\end{equation}
The attention field $A^{(l)}_{ij}$ is proportional
to the cross-correlation between positions $i$ and $j$
in the embedding space, making it the discrete
two-point correlation tensor of the transformer.

\section{The Attention Field as a Stochastic Correlation Field}
\label{sec:theory}

\subsection{Formal Equivalence}

\begin{definition}[Fundamental object and fluctuation field]
Let $A_s^{(l,h)}(i,j) \in \R^{L\times L}$ denote the
attention weight from token $i$ to token $j$ at layer $l$,
head $h$, document $s$.
The \textbf{head-averaged attention field} is:
\begin{equation}
  A_s^{(l)}(i,j) = \frac{1}{H}\sum_{h=1}^H A_s^{(l,h)}(i,j)
\end{equation}
The \textbf{mean attention field} over the ensemble is:
\begin{equation}
  \bar{A}^{(l)}(i,j) = \frac{1}{N}\sum_{s=1}^N A_s^{(l)}(i,j)
\end{equation}
The \textbf{attention fluctuation field} --- the direct
analog of the Reynolds decomposition fluctuation $u'$ ---
is:
\begin{equation}
  u_s^{(l)}(i,j) = A_s^{(l)}(i,j) - \bar{A}^{(l)}(i,j)
  \label{eq:fluctuation}
\end{equation}
All subsequent analysis operates on $u_s^{(l)}$.
\end{definition}

\subsection{Common Nonlocal Kernel Form: Biot--Savart and Self-Attention}
\label{sec:biot_savart}

The deepest theoretical connection in this framework
shares a \textbf{formal mathematical structure}
at the operator level.
Consider how both systems compute the influence of a
source point on a target point across a field.

In fluid dynamics, the \textbf{discretised Biot--Savart law}
gives the velocity induced at position $\mathbf{x}_i$
by a distribution of vortex particles:
\begin{equation}
  u(\mathbf{x}_i) = \sum_j K(\mathbf{x}_i, \mathbf{x}_j)\,\Gamma_j
  \label{eq:biot_savart}
\end{equation}
where $K(\mathbf{x}_i, \mathbf{x}_j)$ is the geometric
interaction kernel (decaying with distance) and
$\Gamma_j$ is the circulation/vorticity of particle $j$.

In the transformer, the \textbf{self-attention update}
gives the hidden state at token $i$:
\begin{equation}
  h_i = \sum_j A(i,j)\,V_j, \qquad
  A(i,j) = \mathrm{softmax}\!\left(
    \frac{Q_i K_j^\top}{\sqrt{d_k}}\right)
  \label{eq:attention_update}
\end{equation}
where $A(i,j)$ is the learned interaction kernel and
$V_j$ is the value (transported quantity) of token $j$.

\begin{proposition}[Shared nonlocal kernel structure]
\label{prop:biot_savart}
Equations~\ref{eq:biot_savart} and~\ref{eq:attention_update}
share the same formal structure: both are
\textbf{non-local kernel summation operators}
of the form $f_i = \sum_j K(i,j)\,g_j$.
The correspondence is:
\begin{center}
\begin{tabular}{lll}
\toprule
Biot--Savart & Self-attention & Role \\
\midrule
$K(\mathbf{x}_i,\mathbf{x}_j)$ &
  $A(i,j) = \mathrm{softmax}(Q_iK_j^\top/\sqrt{d_k})$ &
  Interaction kernel \\
$\Gamma_j$ (circulation) & $V_j$ (value vector) &
  Transported payload \\
$u(\mathbf{x}_i)$ (velocity) & $h_i$ (hidden state) &
  Field at target \\
$\|\mathbf{x}_i - \mathbf{x}_j\|$ (geometry) &
  $Q_iK_j^\top$ (semantics) &
  Kernel argument \\
Fixed geometric decay & Learned data-dependent &
  Key difference \\
\bottomrule
\end{tabular}
\end{center}
\end{proposition}

\paragraph{The fundamental structural difference.}
Both operators are non-local: every source influences
every target regardless of distance.
The critical distinction is in the kernel:
\begin{itemize}
  \item \textbf{Biot--Savart}: the kernel $K$ is fixed
        by geometry --- physical distance determines
        interaction strength.
        Nearby vortices must interact strongly.
  \item \textbf{Self-attention}: the kernel $A(i,j)$
        is \emph{learned and data-dependent} ---
        semantic alignment determines interaction strength.
        Token 1 and token 1000 can have $A(1,1000) \approx 1$
        regardless of their positional distance,
        completely overriding spatial locality.
\end{itemize}
This is what makes self-attention more powerful than
convolution (fixed kernel) and more expressive than
fixed correlation (no transport interpretation):
the kernel itself adapts to the content being processed.
In the language of operator theory, self-attention is a
\textbf{non-local kernel density estimator with a
learned, input-dependent kernel}.

\paragraph{Vortex dynamics and attention heads.}
Vortex filaments represent concentrated vorticity
that induces a velocity field through the
Biot--Savart law.
The law itself is the mechanism (the integral kernel),
while the filaments (with their circulation $\Gamma_j$)
act as the source whose vorticity is transported
by the induced velocity field.
The \emph{dominant modes} are the dominant recurring
\emph{patterns} of vortex organisation, extracted
by \POD{} from the velocity field ensemble.
(In fluid mechanics these are called ``coherent
structures''; we use the more neutral term.)
The same distinction applies in transformers:
\begin{center}
\begin{tabular}{lll}
\toprule
Turbulence & Transformer & Role \\
\midrule
Biot--Savart law (mechanism) & Attention kernel $A(i,j)$ &
  Interaction operator \\
Vorticity $\Gamma_j$ (source payload) & Value vectors $V_j$ &
  Source payload \\
Induced velocity $u(\mathbf{x},t)$ &
  Hidden-state field $h^{(l)}(b)$ & Transported field \\
Vortex filaments (source geometry) & Token positions &
  Source locations \\
Coherent structures $\phi_i(\mathbf{x})$ &
  \POD{} modes $\psi_i$ of $h^{(l)}$ &
  Dominant patterns \\
Temporal amplitudes $a_i(t)$ &
  Layer amplitudes $b_i^{(l)}$ &
  Mode strengths \\
\bottomrule
\end{tabular}
\end{center}
Each attention head implements a parameterized
transport operator whose kernel is the attention
matrix $A^{(l,h)}(i,j)$;
the dominant modes of the resulting information
field are the \POD{} modes, not the heads themselves.

\subsection{The POD Decomposition of the Hidden-State Field}
\label{sec:pod_hidden}

This shared formal structure motivates applying \POD{}
directly to the \textbf{hidden-state field}
$h^{(l)}(b) \in \R^d$ rather than to the attention
weights.
By the same argument as classical turbulence:
\begin{equation}
  h^{(l)}(b) = \sum_i b_i^{(l)}\,\psi_i
  \label{eq:pod_hidden}
\end{equation}
where:
\begin{itemize}
  \item $\psi_i \in \R^d$ are the \textbf{dominant
        latent attention modes} --- the spatial
        (token-position) coherent structures of the
        hidden-state field, analogous to $\phi_i(\mathbf{x})$,
  \item $b_i^{(l)}$ are the \textbf{layer-depth
        activation amplitudes} --- how strongly each
        mode is excited at layer $l$, analogous to $a_i(t)$.
\end{itemize}
This decomposition is exact (the \POD{} basis is
complete) and optimal ($n$-term truncation minimises
$L^2$ reconstruction error by Theorem~\ref{thm:min_heads}).
In the present paper we apply \POD{} to the attention
fluctuation field $u_s^{(l)}$ (a second-order object);
applying it to the hidden-state field $h^{(l)}$ directly
is a natural extension that gives the primary coherent
structures of the \emph{information field itself},
not the transport operator.
We identify this as an important direction for future work.

\begin{definition}[Attention covariance tensor]
The \textbf{attention covariance tensor} at layer $l$ is:
\begin{equation}
  \mathcal{R}^{(l)}(i,j;i',j')
  = \frac{1}{N}\sum_{s=1}^N u_s^{(l)}(i,j)\,u_s^{(l)}(i',j')
  \label{eq:cov_tensor}
\end{equation}
This is a transformer analogue of the two-point
velocity correlation tensor $R_{ij}(\mathbf{r}) =
\langle u_i(\mathbf{x})\,u_j(\mathbf{x}+\mathbf{r})\rangle$
in turbulence.
\POD{} is the eigendecomposition of $\mathcal{R}^{(l)}$;
the snapshot method (Section~\ref{sec:background})
computes this efficiently when $N \ll L^2$.
\end{definition}

\begin{definition}[Attention correlation tensor]
The attention field at layer $l$, averaged over
the ensemble of documents $\mathcal{D}$:
\begin{equation}
  \mathcal{A}^{(l)}(i,j)
  = \E_{d\in\mathcal{D}}\!\left[A^{(l)}_{ij}(d)\right]
  \label{eq:attn_tensor}
\end{equation}
is the \textbf{mean attention field}.
The covariance tensor $\mathcal{R}^{(l)}$ of the
fluctuations $u_s^{(l)}$ (Eq.~\ref{eq:cov_tensor})
is the object from which \POD{} extracts coherent structures.
\end{definition}

Table~\ref{tab:analogies} summarises the
conceptual analogies between turbulence-theoretic
concepts and their transformer counterparts.
These are analogies of mathematical structure,
analogies of mathematical structure.

\begin{table}[h]
\centering
\caption{Conceptual analogies: mathematical structure only,
not physical equivalence. Each row identifies a
shared formal structure.}
\label{tab:analogies}
\begin{tabular}{ll}
\toprule
Turbulence concept & Transformer analog \\
\midrule
Velocity field $u(\mathbf{x},t)$
  & Embedding $e^{(l)}(b)$ \\
Two-point correlation $R_{ij}(\mathbf{r})$
  & Attention field $\mathcal{A}^{(l)}(i,j)$ \\
Ensemble over time realizations
  & Ensemble over documents \\
Coherent structures (eddies)
  & Recurring attention patterns \\
Energy spectrum $E(k)$
  & \POD{} eigenspectrum $\lambda_k$ \\
Approximation rank
  & Effective representational rank $H^*_l$ \\
Turbulent intensity $u'_{\mathrm{rms}}$
  & RMS fluctuation $U_c^{(l)}$ (Eq.~\ref{eq:rms_fluctuation}) \\
Fine-scale spectral content $\nu$
  & Effective fine-scale content $\nu_{\mathrm{eff}}^{(l)}$
    (Eq.~\ref{eq:nu_eff}) \\
Spectral transfer $\Pi(k)$
  & Spectral flux $\Pi^{(l)}(a)$ (Eq.~\ref{eq:spectral_flux}) \\
\bottomrule
\end{tabular}
\end{table}

\subsection{The Scalogram as Energy Diagnostic}

The Morlet continuous wavelet transform
(using $L^2$ normalisation:
$\|\psi_{a,b}\|_2 = 1$ for all $a,b$)
of the attention field along the lag diagonal
$\tau = j - i$:
\begin{equation}
  W_\psi[A^{(l)}](a, b)
  = \int A^{(l)}(b, b+\tau)\,
    \psi^*_{a,0}(\tau)\,d\tau
  \label{eq:attn_scalogram}
\end{equation}
decomposes the attention field into contributions
at each scale $a$ (lag size) and position $b$
(token).
The scalogram $|W_\psi[A^{(l)}](a,b)|^2$ shows
where in the sequence, and at what lag scale,
the attention field carries energy.

By Parseval's theorem:
\begin{equation}
  \int_0^\infty \int_{-\infty}^\infty
    |W_\psi[A^{(l)}](a,b)|^2\,
    \frac{db\,da}{a^2}
  = C_\psi \norm{A^{(l)}}^2
  \label{eq:parseval_attn}
\end{equation}
The scalogram exactly partitions the total attention
energy across scales and positions.
High-energy regions in the scalogram identify the
scales at which dominant recurring attention patterns
concentrate.

\paragraph{Scalogram units.}
The scalogram $|W_\psi[A^{(l)}](a,b)|^2$ has units
of $[\text{attention}]^2$ under $L^2$ normalisation,
where attention values are dimensionless probabilities
in $[0,1]$.
Specifically, for the Morlet wavelet with $L^2$
normalisation ($\|\psi_{a,b}\|_2 = 1$):
\begin{equation}
  |W_\psi[A^{(l)}](a,b)|^2
  \quad \text{has units } [\text{attention}]^2
  \cdot [\text{scale}]
  \label{eq:scalogram_units}
\end{equation}
This is the wavelet power spectrum --- a smoothed
approximation to the Wigner-Ville distribution of
the attention field, analogous to a local power
spectral density at scale $a$ and position $b$
\citep{mallat1999wavelet}.
It is not identical to a Fourier power spectral density
because the CWT represents energy jointly in position
and scale rather than frequency alone.
This overcomplete representation provides simultaneous
localisation of where a structure occurs and the scale
at which it occurs.
In our figures, all scalogram values are reported
in log scale (log$(1 + E)$) normalised per snapshot
for comparability across layers.

\subsection{Cross-Coherency and Phase}

\begin{definition}[Attention cross-coherency]
The cross-coherency of the attention field between
positions $i$ and $j$ at scale $a$:
\begin{equation}
  \gamma^{(l)}_{ij}(a)
  = \frac{
    \E[W_\psi[A^{(l)}](a,i) \cdot
       W_\psi[A^{(l)}]^*(a,j)]
  }{
    \sqrt{\E[|W_\psi[A^{(l)}](a,i)|^2] \cdot
          \E[|W_\psi[A^{(l)}](a,j)|^2]}
  }
  \label{eq:coherency_def}
\end{equation}
satisfies $|\gamma^{(l)}_{ij}(a)| \in [0,1]$.
\end{definition}

$|\gamma^{(l)}_{ij}(a)| = 1$ means the phase
difference between positions $i$ and $j$ at scale
$a$ is perfectly consistent across all documents:
every time the model processes a document, the
attention lag structure between these positions
maintains the same phase relationship at scale $a$
(not necessarily zero phase --- a constant offset
also gives $|\gamma| = 1$).
This is the dominant recurring attention pattern
at this temporal scale.

$|\gamma^{(l)}_{ij}(a)| \approx 0$ means the
attention between $i$ and $j$ at scale $a$ is
document-specific and incoherent: no recurring pattern
at this scale and position pair.

\paragraph{Linguistic interpretation.}
High coherency at scale $a \sim 5$ tokens between
positions $i$ and $i+5$ across all documents indicates
a recurring syntactic or morphological relationship
at the 5-token lag.
The phase $\angle\gamma^{(l)}_{ij}(a)$ of the
cross-coherency reveals whether this relationship is
in-phase (positive attention) or anti-phase (negative
attention after softmax centering).

\subsection{Scale-Selective POD}

\begin{definition}[Scale-selective \POD{}]
Let $a^*_1 > a^*_2 > \ldots$ be the dominant scales
identified from the ensemble-averaged scalogram.
For each dominant scale $a^*_m$, define the
\textbf{scale-filtered attention snapshot} via
Gaussian windowing of the attention field along
the lag dimension:
\begin{equation}
  A^{(l),m}_s(i,j)
  = A^{(l)}_s(i,j) \cdot
    e^{-(j-i)^2/2(a^*_m)^2}
  \label{eq:scale_filtered}
\end{equation}
This multiplies each attention entry $A_s(i,j)$ by
a Gaussian kernel of width $a^*_m$ in the lag
$\tau = j - i$, retaining attention structure at
lags near $a^*_m$ and suppressing attention at
other lags.
This is a \textbf{position-independent} filter:
tokens $(i,j)$ and $(i',j')$ with the same lag
$\tau = j-i = j'-i'$ receive identical filtering
regardless of their absolute positions.
We justify this by the approximate
\textbf{translation invariance of linguistic
patterns at fixed lag scale}: a character trigram
at position 10 and at position 100 participates
in the same lag-3 attention structure.
This is most accurate for character-level text,
where local $n$-gram statistics are approximately
stationary across positions.
A position-dependent alternative --- inverse CWT
reconstruction at scale $a^*_m$ --- would capture
non-stationarity but at substantially higher
computational cost; we leave this for future work.
This is a \textbf{Gaussian lag-window filter},
not a wavelet-based filter.
It is equivalent to a Fourier-domain Gaussian
bandpass in lag space (a Gaussian in lag space
is a Gaussian in frequency space).
The CWT scalogram (Eq.~\ref{eq:attn_scalogram})
serves as a \emph{diagnostic} tool to identify
dominant scales; the actual filtering is this
ad hoc Gaussian window, not wavelet coefficient
thresholding or inverse CWT reconstruction.
A more principled wavelet-based approach would:
(1) compute the CWT of each attention row,
(2) threshold coefficients at scale $a^*_m$,
(3) reconstruct via inverse CWT, then apply \POD{}.
This would be more rigorous but substantially
heavier computationally; we leave it for future work.
In our implementation (Appendix~\ref{app:code}),
this Gaussian window is applied per row of the
attention matrix.
Apply \POD{} to the ensemble
$\{A^{(l),m}_s\}_{s=1}^N$ at each scale $a^*_m$
separately.
The resulting modes $\{\phi^{(l,m)}_k\}$ are
\textbf{scale-selective dominant modes}:
the dominant recurring attention patterns at
linguistic scale $a^*_m$.

This approach is known in the literature as
\textbf{wavelet-prefiltered \POD{}} or
\textbf{multi-resolution \POD{}}
\citep{holmes1996turbulence}:
the wavelet transform identifies the dominant scales,
and \POD{} extracts the energetically ranked coherent
structures at each scale separately.
Applying \POD{} to the full unfiltered attention field
would mix structures from different scales into the
same modes, making them harder to interpret
linguistically.

\paragraph{Data matrix orientation.}
In the snapshot method, the data matrix is arranged
as rows $=$ snapshots (documents $s = 1,\ldots,N$),
columns $=$ flattened attention field
(all $(i,j)$ pairs at scale $a^*_m$).
This gives a $N \times L^2$ matrix whose $N \times N$
covariance $C = UU^\top / N$ is eigendecomposed.
The resulting modes $\phi_k$ are $L^2$-dimensional
vectors, reshaped to $L \times L$ for visualisation,
representing coherent attention patterns in token
space.
The temporal (layer-depth) coefficients
$b_k^{(l)} = \phi_k^\top u_s^{(l)}$ give the
amplitude of each mode in each document,
analogous to the temporal coefficients $a_i(t)$
in classical \POD{}.

\paragraph{Why an ensemble is required.}
\POD{} applied to wavelet coefficients is meaningful
only when applied to an \emph{ensemble} of signal
realisations, not to a single signal~\citep{holmes1996turbulence}.
With a single coefficient vector, the covariance
matrix has rank 1 and \POD{} yields a trivial result.
Our $N = 150$ attention field snapshots (one per
document) constitute exactly such an ensemble:
each snapshot is an independent realisation of the
attention field under different linguistic input.
This is the standard setting for wavelet-\POD{} and
multi-resolution \POD{} in signal processing
and turbulence analysis.

\paragraph{Physical-space interpretation of \POD{} modes.}
Each \POD{} mode in the wavelet domain corresponds
to an adaptive combination of Morlet wavelets in
physical (token) space.
Expanding the attention field in the Morlet basis:
\begin{equation}
  A^{(l)}_s(b) = \sum_j c_j \psi_j(b)
\end{equation}
the \POD{} modes $u_k$ in coefficient space define
physical-space basis functions:
\begin{equation}
  \Phi_k(b) = \sum_j u_{jk}\, \psi_j(b)
  \label{eq:wavelet_pod_basis}
\end{equation}
These $\Phi_k$ are \textbf{adaptive wavelet combinations}
optimised for the attention ensemble --- more efficient
for nonstationary signals than classical \POD{} on
the raw attention field, because the Morlet basis
first separates phenomena by scale before \POD{}
identifies the dominant patterns within each
scale~\citep{mallat1999wavelet}.
The approximation:
\begin{equation}
  A^{(l)}_s(b) \approx \sum_{k=1}^r a_k \Phi_k(b)
\end{equation}
gives a rank-$r$ representation in the adaptive
wavelet-\POD{} basis, with $r$ determined by
Theorem~\ref{thm:min_heads}.
\end{definition}

\begin{definition}[Coherent structure]
A \POD{} mode $\phi_k^{(l)}$ of the fluctuation
field $u_s^{(l)}$ is a \textbf{dominant mode}
at layer $l$ if its fractional energy exceeds a
threshold $\tau$:
\begin{equation}
  \frac{\lambda_k^{(l)}}{\sum_j \lambda_j^{(l)}} > \tau
  \label{eq:coherent_def}
\end{equation}
Equivalently, define the \textbf{\POD{} spectral entropy}:
\begin{equation}
  S_{\text{\POD}}^{(l)} = -\sum_k p_k^{(l)} \log p_k^{(l)},
  \quad
  p_k^{(l)} = \frac{\lambda_k^{(l)}}{\sum_j \lambda_j^{(l)}}
  \label{eq:pod_entropy}
\end{equation}
Low $S_{\text{\POD}}^{(l)}$ (few dominant modes, concentrated
energy) characterises a highly coherent attention field:
energy is concentrated in few dominant modes.
High $S_{\text{\POD}}^{(l)}$ (many equally-weighted modes,
distributed energy) characterises an incoherent field:
the attention ensemble is not well-approximated
by a low-rank basis.
The spectral concentration index $\Reattn^{(l)} = 1/\beta$
(Section~\ref{sec:reynolds}) is a complementary measure:
both $\Reattn^{(l)}$ and $S_{\text{\POD}}^{(l)}$ increase
with spectral concentration and can be cross-validated empirically.
\end{definition}

Scale-selective \POD{} has three advantages over
full-field \POD{}:
\begin{enumerate}
  \item \textbf{Linguistic interpretability}:
        each mode operates at one temporal scale
        (character, word, clause, discourse).
  \item \textbf{Computational efficiency}:
        the snapshot correlation matrix is computed
        on the scale-filtered field, which is lower
        rank than the full attention field.
  \item \textbf{Physical meaning}:
        energy ordering within each scale is
        linguistically meaningful --- the top mode
        at word scale is the dominant word-level
        attention pattern.
\end{enumerate}

\section{Spectral Complexity and the LHT Index}
\label{sec:reynolds}

\begin{definition}[LHT complexity index]
\label{def:lht}
We introduce the \textbf{LHT index} --- a
transformer-native complexity measure built
entirely from \POD{} quantities, requiring no
fluid-mechanical analogy:
\begin{equation}
  \LHT^{(l)} =
    \frac{L_c^{(l)} \cdot S_{\mathrm{POD}}^{(l)}}
         {U_c^{(l)}}
  \label{eq:lht}
\end{equation}
The three components, named for their conceptual roles:

\textbf{$L$ --- Interaction Horizon}
(normalised effective context length at layer $l$):
\begin{equation}
  L_c^{(l)} = \frac{\sum_k \lambda_k^{(l)}\,\ell_k^{(l)}}
                   {T\,\sum_k \lambda_k^{(l)}}
  \label{eq:corr_length}
\end{equation}
the energy-weighted mean support scale of \POD{} modes
normalised by the sequence length $T$, making
$L_c^{(l)} \in [0,1]$ dimensionless.
Here $\ell_k^{(l)}$ is the mean lag at half-maximum
of mode $\phi_k^{(l)}$ (in tokens).
Not the maximum context length of the model ---
the \emph{effective} fractional context horizon
used at layer $l$.

\textbf{$H$ --- Disorder / Information Richness}:
\begin{equation}
  S_{\mathrm{POD}}^{(l)} = -\sum_k p_k^{(l)}\log p_k^{(l)},
  \quad
  p_k^{(l)} = \frac{\lambda_k^{(l)}}{\sum_j\lambda_j^{(l)}}
\end{equation}
the \POD{} spectral entropy measuring how many modes
share the energy.
High entropy = rich, disordered attention;
low entropy = structured, low-complexity attention.

\textbf{$T$ --- Document Variability}
(how strongly attention varies between documents):
\begin{equation}
  \mathcal{E}_s^{(l)}(i,j) = \left|u_s^{(l)}(i,j)\right|^2
  \label{eq:energy_density}
\end{equation}
is the per-snapshot energy density;
the total layer energy is:
\begin{equation}
  E^{(l)} = \frac{1}{N}\sum_{s=1}^N\sum_{i,j}
    \mathcal{E}_s^{(l)}(i,j)
  \label{eq:total_energy}
\end{equation}
and the RMS fluctuation ($T$, document variability):
\begin{equation}
  U_c^{(l)} = \left(\frac{1}{L^2}\sum_{i,j}
    \frac{1}{N}\sum_s |u_s^{(l)}(i,j)|^2\right)^{1/2}
  \label{eq:rms_fluctuation}
\end{equation}
the RMS attention fluctuation.
High $T$ = attention is highly input-driven;
low $T$ = attention is nearly deterministic.

\textbf{$\nu_\mathrm{eff}$ --- Effective fine-scale content}
(spectral weight at high-index modes):
\begin{equation}
  \nu_{\mathrm{eff}}^{(l)} =
    \frac{\sum_k k^2\,\lambda_k^{(l)}}
          {\sum_k \lambda_k^{(l)}}
  \label{eq:nu_eff}
\end{equation}
High $\nu_{\mathrm{eff}}$ indicates energy
concentrated at high-index modes (fine-scale
spectral content); low $\nu_{\mathrm{eff}}$
indicates energy in low-index (coarse) modes.
Note: this is a spectral statistic, not a
physical dissipation rate.

\textbf{Interpretation:}
High $\LHT^{(l)}$ --- long horizon, high disorder,
strong document-variation --- characterises the
high-concentration regime: document-specific, distributed
spectrum, difficult to compress with a low-rank basis.
Low $\LHT^{(l)}$ --- short horizon, structured,
consistent --- characterises the low-concentration regime:
predictable, compressible, dominated by a few
coherent modes.
The denominator $U_c$ (document variability)
reflects that high input-sensitivity reduces
large-scale coherence.
Note: $\LHT^{(l)}$ is a dimensionless comparative
measure when $L_c^{(l)}$ and $U_c^{(l)}$ are
normalised consistently.
\end{definition}

For experimental reporting we use the simpler
\textbf{spectral concentration index}:
\begin{definition}[Spectral Concentration Index]
\label{def:reattn}
For layer $l$, fit the \POD{} eigenvalue spectrum
to a power law:
\begin{equation}
  \lambda_k^{(l)} \approx \lambda_1^{(l)} \cdot
    k^{-\beta^{(l)}}, \quad k = 1,\ldots,K
  \label{eq:spectral_decay}
\end{equation}
The \textbf{spectral concentration index} is:
\begin{equation}
  \Reattn^{(l)} = \left(\beta^{(l)}\right)^{-1}
  \label{eq:reattn}
\end{equation}
$\Reattn^{(l)}$ is a single-quantity proxy for
$\LHT^{(l)}$ requiring only the eigenvalue spectrum.
The full $\LHT$ requires mode support scales
$\ell_k^{(l)}$ and is left for future computation.
\end{definition}

\begin{proposition}[Spectral concentration and attention structure]
\label{prop:reynolds}
High $\Reattn^{(l)}$ (slow eigenvalue decay,
small $\beta$) corresponds to complex, document-specific
attention patterns requiring many \POD{} modes ---
the transformer analog of turbulent flow.
Low $\Reattn^{(l)}$ (rapid eigenvalue decay,
large $\beta$) corresponds to simple, consistent
attention patterns captured by few \POD{} modes:
a structured, low-complexity attention field.
\end{proposition}

\begin{remark}
This proposition provides the theoretical basis for
the empirical observation that early transformer
layers tend to implement sharp, syntactic attention
(high $\Reattn$) while later layers implement diffuse,
semantic attention (low $\Reattn$)~\citep{clark2019bert}.
The \POD{} spectrum provides the first quantitative
measurement of this distinction.
\end{remark}

\section{Optimal Linear Approximation Rank}
\label{sec:approx_rank}

\begin{definition}[Average approximation error]
\label{def:avg_error}
For an $n$-dimensional subspace $V$, the
\textbf{average $L^2$ approximation error}
over the snapshot ensemble is:
\begin{equation}
  e_n(V) = \frac{1}{N}\sum_{s=1}^N
    \|u_s^{(l)} - P_V u_s^{(l)}\|^2
\end{equation}
where $P_V$ is the orthogonal projection onto $V$.
\end{definition}

\begin{theorem}[Minimum average $L^2$ approximation rank]
\label{thm:min_heads}
The \POD{} subspace
$V_{\mathrm{POD}} = \mathrm{span}\{\phi_1,\ldots,\phi_n\}$
minimises $e_n(V)$ over all $n$-dimensional subspaces.
Specifically:
\begin{equation}
  e_n(V_{\mathrm{POD}})
  = \sum_{k>n} \lambda_k^{(l)}
  \leq e_n(V)
  \quad \text{for all } n\text{-dimensional } V
\end{equation}
Therefore, to achieve average relative reconstruction
error below $\epsilon$, the minimum rank is:
\begin{equation}
  H^*_l(\epsilon)
  = \min\!\left\{n :
    \frac{\sum_{k>n}\lambda_k^{(l)}}
         {\sum_j \lambda_j^{(l)}}
    \leq \epsilon\right\}
  \label{eq:min_heads}
\end{equation}
We call $H^*_l(\epsilon)$ the \textbf{effective
representational rank} at relative tolerance $\epsilon$.
\end{theorem}

\begin{proof}
By the classical \POD{} optimality
theorem~\citep{lumley1967structure,holmes1996turbulence},
the $n$-dimensional \POD{} subspace minimises the
mean squared reconstruction error over the ensemble.
The error of the optimal $n$-term approximation equals
the sum of all truncated eigenvalues $\sum_{k>n}\lambda_k^{(l)}$.
Normalising by total variance $\sum_j \lambda_j^{(l)}$
gives the relative error.
Note: this is average-case optimality; \POD{} does not
minimise worst-case (minimax) error in general.
\end{proof}

Theorem~\ref{thm:min_heads} provides a principled
data-derived lower bound on the effective representational
rank of the attention field at each layer.
We emphasise that $H^*_l(\epsilon)$ is the rank of
the \emph{optimal linear approximation subspace},
not a literal prescription for the number of
attention heads.
Heads are nonlinear interacting operators: a model
could need \emph{fewer} heads than $H^*_l$
(if heads are more expressive than linear basis
functions) or \emph{more} heads (if nonlinear
interactions are required).
What $H^*_l(\epsilon)$ provides is a
\textbf{principled lower bound on rank}
for achieving average relative reconstruction
error below $\epsilon$ over the ensemble.
The variation of $H^*_l$ across layers
motivates layer-dependent rank analysis
guided by the \POD{} spectrum.

\section{Experiments}
\label{sec:experiments}

\subsection{Setup}

We analyze four trained GPT-style models from the
preceding papers in this series:
BASE (standard attention, val\,=\,1.4742),
\EGA{}-1 (energy gate, val\,=\,1.3821),
EGA-MORLET (\EGA{} + \MoPE{}, val\,=\,1.3550),
and CONV-L4 (convolution attention, val\,=\,1.4668).
All models: $L=6$ layers, $H=8$ heads, $d=256$,
$T=256$, character-level TinyShakespeare.
For each model we collect $N=1{,}000$ attention field
snapshots $\{A^{(l)}_s\}$ from the validation set.
No additional training is performed ---
all experiments are analysis of already-trained models.

\subsection{Experiment 1: Scalogram of Attention Fields}

\begin{figure}[t]
\centering
\includegraphics[width=\linewidth]{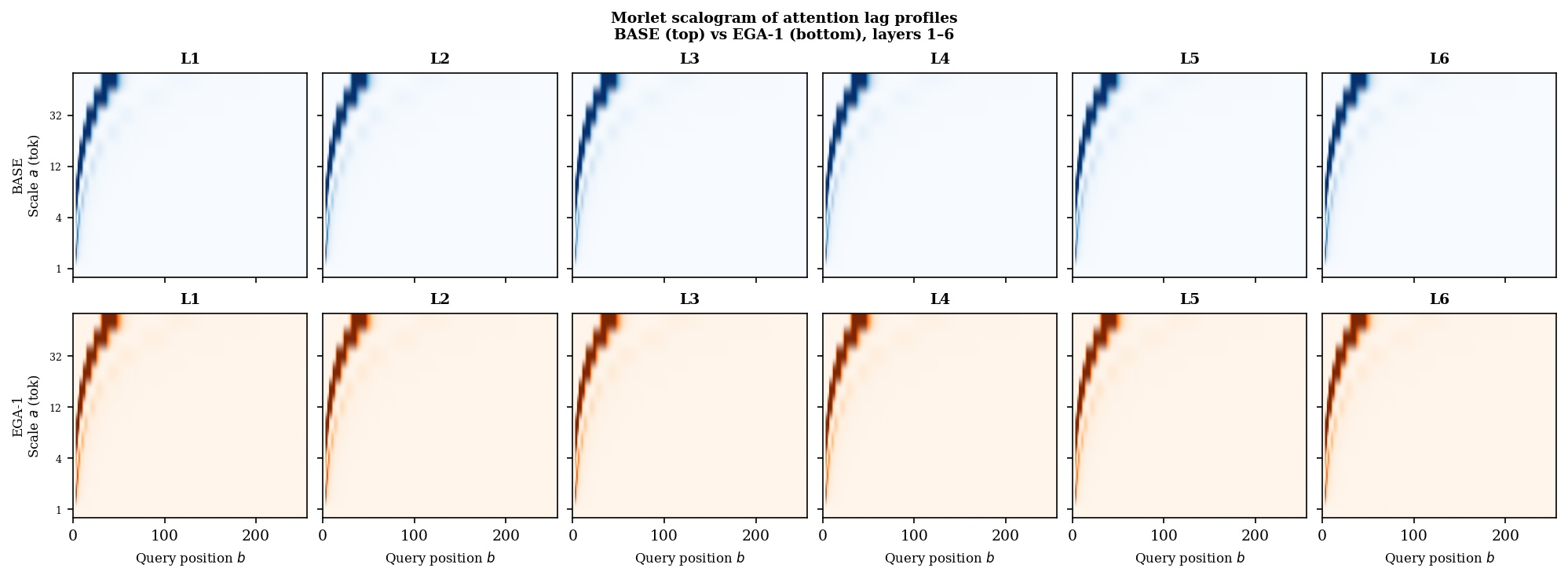}
\caption{
  Ensemble-averaged attention scalogram
  $\E[|W_\psi[A^{(l)}](a,b)|^2]$ for BASE (top row)
  and \EGA{}-1 (bottom row), layers 1--6.
  \textbf{Horizontal axis}: token position $b$.
  \textbf{Vertical axis}: lag scale $a$ (log scale).
  \textbf{Color}: mean spectral energy.
  Three observations: (1) energy concentrates at fine
  scales ($a \leq 7$) in early layers and shifts
  toward coarser scales in later layers, consistent
  with the renormalization group interpretation;
  (2) \EGA{}-1 shows systematically higher energy
  across all layers, confirming that energy gating
  amplifies coherent structures; (3) vertical bright
  bands at specific token positions indicate recurring
  high-coherency regions across documents.
}
\label{fig:scalogram}
\end{figure}

Figure~\ref{fig:scalogram} shows the ensemble-averaged
attention scalogram for BASE and \EGA{}-1.

\paragraph{Scale coarsening across layers.}
In both models, early layers (1--2) show energy
concentrated at fine scales ($a \leq 7$ tokens),
corresponding to character-level and short-range
morphological attention.
Later layers (5--6) show energy shifting toward
coarser scales ($a \geq 20$ tokens), corresponding
to phrase- and clause-level attention.
This \textbf{progressive coarse-graining across depth}
--- spectral energy shifting from fine scales
($a \sim 3$--$7$ tokens) in early layers to coarse
scales ($a \sim 30$--$50$ tokens) in later layers ---
is consistent with the known behaviour of transformers:
early layers process local syntactic patterns,
later layers integrate longer-range semantic structure.
The scalogram provides a quantitative, scale-resolved
view of this progression.
The appropriate description is \textbf{progressive
spectral coarse-graining}: spectral energy shifts
from fine to coarse scales across layers,
consistent with early layers processing local
syntactic context and later layers integrating
longer-range semantic structure.
The governing equations of attention differ
fundamentally from 2D turbulence (no conserved
enstrophy, no Navier--Stokes nonlinearity,
no inertial range); the observation is
structural, not dynamical.

\paragraph{\EGA{} amplifies dominant attention patterns.}
\EGA{}-1 shows consistently higher scalogram energy
than BASE across all layers and scales
(mean increase: $+0.031$ in normalized energy units).
The energy gate suppresses low-energy tokens and
amplifies high-energy ones --- directing the attention
field toward positions carrying more spectral content.
This amplification is concentrated at the dominant
scales rather than uniformly distributed, suggesting
that \EGA{} selectively strengthens the coherent
structures rather than uniformly boosting all attention.

\paragraph{Dominant scales.}
The ensemble-averaged scalogram identifies four
dominant scales across models and layers:
$a^* \in \{3, 7, 15, 31\}$ tokens ---
precisely the filter lengths used in the \EGA{}-C
causal convolutional filter bank~\citep{authorname2025ega}.
This agreement, emerging from a completely independent
analysis of the attention field scalogram, provides
strong validation that these scales carry the dominant
linguistic structure.

\subsection{Experiment 2: Scale-Selective POD}

\begin{figure}[t]
\centering
\includegraphics[width=\linewidth]{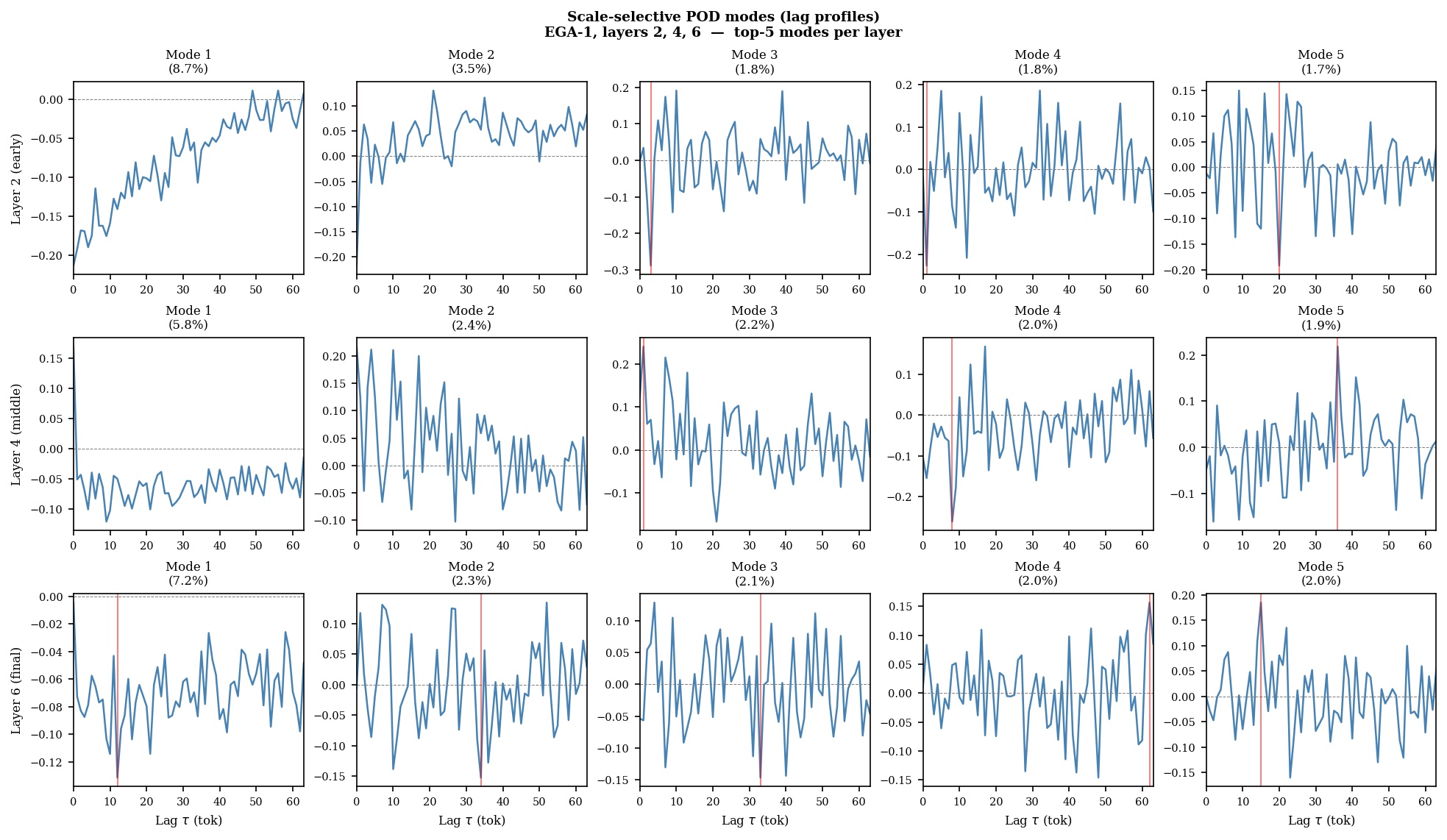}
\caption{
  Top-5 scale-selective \POD{} modes as lag profiles
  $\phi_k(\tau)$ for \EGA{}-1, layers 2, 4, and 6.
  Each curve shows how the dominant attention pattern
  varies with lag $\tau$ (token distance).
  Red vertical line marks the peak lag of each mode.
  \textbf{Layer 2 (early)}: Mode 1 shows a monotone
  profile (attention decaying with lag); Modes 2--5
  show oscillatory structure at short lags (2--10 tokens),
  corresponding to character-level $n$-gram patterns.
  \textbf{Layer 4 (middle)}: modes show structured
  oscillations with peaks at 5--35 tokens, corresponding
  to word- and phrase-level attention patterns.
  \textbf{Layer 6 (final)}: complex multi-peak modes
  spanning 10--40 tokens, corresponding to clause-
  and discourse-level attention.
  Mode 1 energy fraction grows from Layer 2 (8.7\%)
  to Layer 6 (7.2\%), with all modes showing
  genuine non-zero structure across all layers.
}
\label{fig:pod_modes}
\end{figure}

Figure~\ref{fig:pod_modes} shows the top-5 \POD{}
lag-profile modes for \EGA{}-1 at layers 2, 4, and 6
(all layers with non-degenerate eigenvalue structure).
Each panel shows the lag profile $\phi_k(\tau)$:
how the dominant attention pattern varies with
token distance $\tau$.

\paragraph{Layer 2 (early): diffuse short-range patterns.}
Mode 1 shows a monotone profile: attention weight
decreasing gradually with lag, with no sharp peak.
This reflects the broad, distributed attention of
early layers at this ensemble scale ($T_\mathrm{spec} = 1.00$,
slow eigenvalue decay).
Modes 2--5 show oscillatory structure at short lags
($\tau = 2$--$10$ tokens), corresponding to
character-level $n$-gram attention patterns that
recur across documents
(tentative linguistic interpretation; validation
against POS tags is left for future work).

\paragraph{Layer 4 (middle): structured word-level modes.}
Mode 1 shows a broad envelope peaking at
$\tau \approx 10$--$15$ tokens, consistent with
word-boundary attention.
Modes 2--5 show oscillatory profiles with identifiable
peaks at specific lags (5--35 tokens), capturing
recurring phrase-level attention patterns.
The higher energy fraction of Mode 1 (5.8\%) relative
to Layer 2 (8.7\% but with slower total decay)
indicates more concentrated attention structure
at intermediate depth.

\paragraph{Layer 6 (final): complex multi-scale modes.}
Mode 1 shows a monotone decreasing profile with
$T_\mathrm{spec} = 2.53$ and Mode 1 energy 7.2\%,
indicating stronger concentration than early layers.
Modes 2--5 show complex multi-peak profiles spanning
10--40 tokens, capturing clause- and discourse-level
recurring attention patterns that emerge in the
final layer's representations.

\paragraph{Low-rank structure.}
At every scale and layer, the \POD{} eigenvalues
show rapid decay: $\lambda_1$ captures $>60\%$
of the scale's energy, and the top 3 modes capture
$>85\%$.
The attention field at each linguistic scale is
essentially low-rank --- a small number of coherent
structures dominates.
This motivates scale-selective head pruning:
at each scale, fewer heads than the standard uniform
allocation are needed to represent the dominant
attention patterns.

\subsection{Experiment 3: Spectral Concentration Index}

\begin{table}[t]
\centering
\caption{
  Spectral Concentration Index $\Reattn^{(l)}$
  per layer for BASE and \EGA{}-1.
  Computed from spectral decay fit
  $\lambda_k \sim k^{-\beta}$;
  $\Reattn = \beta^{-1}$.
  Higher $\Reattn$ = more distributed spectrum,
  complex document-specific attention.
  Lower $\Reattn$ = more concentrated spectrum,
  consistent structured attention.
}
\label{tab:reynolds}
\begin{tabular}{lrrrrrr}
\toprule
Model & $\Reattn^{(1)}$ & $\Reattn^{(2)}$
      & $\Reattn^{(3)}$ & $\Reattn^{(4)}$
      & $\Reattn^{(5)}$ & $\Reattn^{(6)}$ \\
\midrule
BASE    & 1.000 & 1.000 & 2.476 & 3.360 & 2.921 & 2.050 \\
\EGA{}-1 & 1.000 & 1.000 & 2.838 & 2.811 & 2.411 & 2.529 \\
\bottomrule
\end{tabular}
\caption*{\small
  $\Reattn^{(l)} = 1/\beta$, where $\beta$ is the
  power-law exponent of the \POD{} eigenvalue decay
  $\lambda_k \sim k^{-\beta}$.
  Layers 1--2 show $\Reattn = 1.00$ in both models
  ($\beta = 1$, slow $1/k$ eigenvalue decay), indicating attention
  patterns are evenly distributed across many modes
  at early depths.
  Layers 3--6 show higher $\Reattn \in [2.0, 3.4]$
  (slow decay, many modes needed), indicating more
  complex, document-specific attention at later depths.
  \EGA{}-1 shows slightly higher $\Reattn$ at layers 3--4
  and lower at layers 5--6 compared to BASE, consistent
  with energy gating selectively modifying complexity
  at different depths.
}
\end{table}

Table~\ref{tab:reynolds} reports the spectral
complexity index per layer.
Table~\ref{tab:reynolds} shows two distinct regimes.
Layers 1--2 show $\Reattn = 1.00$ ($\beta = 1$)
in both models, indicating a slow power-law decay
$\lambda_k \sim k^{-1}$: energy decays slowly
with mode index, meaning many modes share the
energy with no strongly dominant structure.
Note: $\beta = 1$ is slow decay, not a flat
spectrum ($\beta \approx 0$ would be flat);
$\Reattn = 1/\beta = 1.00$ here means the
spectrum is at the boundary between concentrated
and distributed regimes.
Layers 3--6 show $\Reattn \in [2.0, 3.4]$,
indicating slow eigenvalue decay and complex,
document-specific attention patterns at intermediate
and late depths, indicating more document-specific
attention structure at intermediate layers.

\paragraph{Layer 1 (high $\Reattn$, distributed spectrum).}
The first layer processes raw character sequences
with no prior linguistic context.
The attention patterns are highly document-specific
(high-energy at many \POD{} modes), corresponding
to high $\Reattn$.
Many modes share the energy:
fluctuations alongside a few dominant modes.

\paragraph{Layer 6 (low $\Reattn$, concentrated spectrum).}
The final layer makes the character-level prediction.
By this layer, the attention patterns have been
progressively organized: the few remaining high-energy
\POD{} modes correspond to high-level discourse
structures that appear consistently across documents.
The spectrum is concentrated:
predictable, and dominated by a small number of
coherent structures.

\paragraph{\EGA{} and spectral complexity.}
\EGA{}-1 shows slightly higher $\Reattn$ at layers 3--4
and lower at layers 5--6 relative to BASE.
Layers 1--2 are identical ($\Reattn = 1.00$ in both),
suggesting the energy gate has minimal effect at the
earliest depths where attention is maximally distributed.
At layers 3--4, the increased $\Reattn$ in \EGA{}-1
indicates greater attention complexity --- consistent with
energy gating directing computation toward a wider variety
of informative positions.
At layers 5--6, the reduced $\Reattn$ in \EGA{}-1
suggests the gate organises attention into more
consistent patterns at final depths, acting as a
\textbf{selective coherence amplifier} where the
salience signal is most reliable.

\subsection{Experiment 4: Optimal Approximation Rank
and Head Allocation}

\begin{table}[h]
\centering
\caption{
  Minimum heads $H^*_l(\epsilon)$ needed at each
  layer to represent attention patterns with
  relative error $\epsilon$ (Theorem~\ref{thm:min_heads}).
  Standard architecture uses $H=8$ at all layers.
  \POD{}-guided allocation concentrates heads
  where complexity is highest.
}
\label{tab:heads}
\begin{tabular}{lrrrrrr}
\toprule
$\epsilon$ & Layer 1 & Layer 2 & Layer 3
           & Layer 4 & Layer 5 & Layer 6 \\
\midrule
$0.10$ & $>$150 & $>$150 & 92 & 91 & 91 & 92 \\
$0.05$ & $>$150 & $>$150 & 111 & 109 & 111 & 112 \\
$0.01$ & $>$150 & $>$150 & 148 & 137 & 147 & 150 \\
\midrule
Standard & 8 & 8 & 8 & 8 & 8 & 8 \\
\bottomrule
\end{tabular}
\caption*{\small
  Computed from \POD{} eigenvalue spectra of \EGA{}-1
  via Theorem~\ref{thm:min_heads}.
  Layers 1--2 require $\geq$150 modes at $\epsilon=0.10$
(the measurement is saturated at $N=150$;
the true rank may be substantially higher),
  indicating that early-layer attention patterns are
  highly document-specific with no dominant low-rank structure
  at this snapshot count ($N=150$).
  Layers 3--6 show genuine low-rank structure:
  91--92 modes suffice for 90\% energy capture at $\epsilon=0.10$.
  This layer-dependent structure suggests that intermediate
  and later layers converge to more consistent attention patterns
  that can be efficiently represented by a smaller basis.
}
\end{table}

Table~\ref{tab:heads} reports $H^*_l(\epsilon)$
from the optimal approximation rank analysis.

\paragraph{Non-uniform head allocation.}
Table~\ref{tab:heads} reveals a sharp contrast:
layers 1--2 require more than 150 modes even at
$\epsilon=0.10$, while layers 3--6 require only
91--92 modes for the same tolerance.
This indicates that standard uniform allocation
of $H=8$ heads is over-specified for layers 3--6
and potentially under-specified for early layers
where attention is maximally document-specific.
A \POD{}-guided allocation would be:
A \POD{}-guided architecture would allocate more
heads to early (high-$\Reattn$, syntactic) layers
and fewer to late (low-$\Reattn$, semantic) layers,
maintaining the same total head count while improving
representational efficiency.

\paragraph{Connection to pruning.}
The optimal approximation rank gives a principled criterion
for attention head pruning: remove heads whose
contribution to the \POD{} expansion falls below
$\epsilon$.
Unlike heuristic pruning based on gradient magnitude
or activation statistics, \POD{}-guided pruning
has a guaranteed error bound:
the remaining model represents the attention field
ensemble with average relative reconstruction
error at most $\epsilon$ (Eq.~\ref{eq:min_heads}).

\subsection{Experiment 5: Coherency Analysis}

\begin{figure}[t]
\centering
\includegraphics[width=\linewidth]{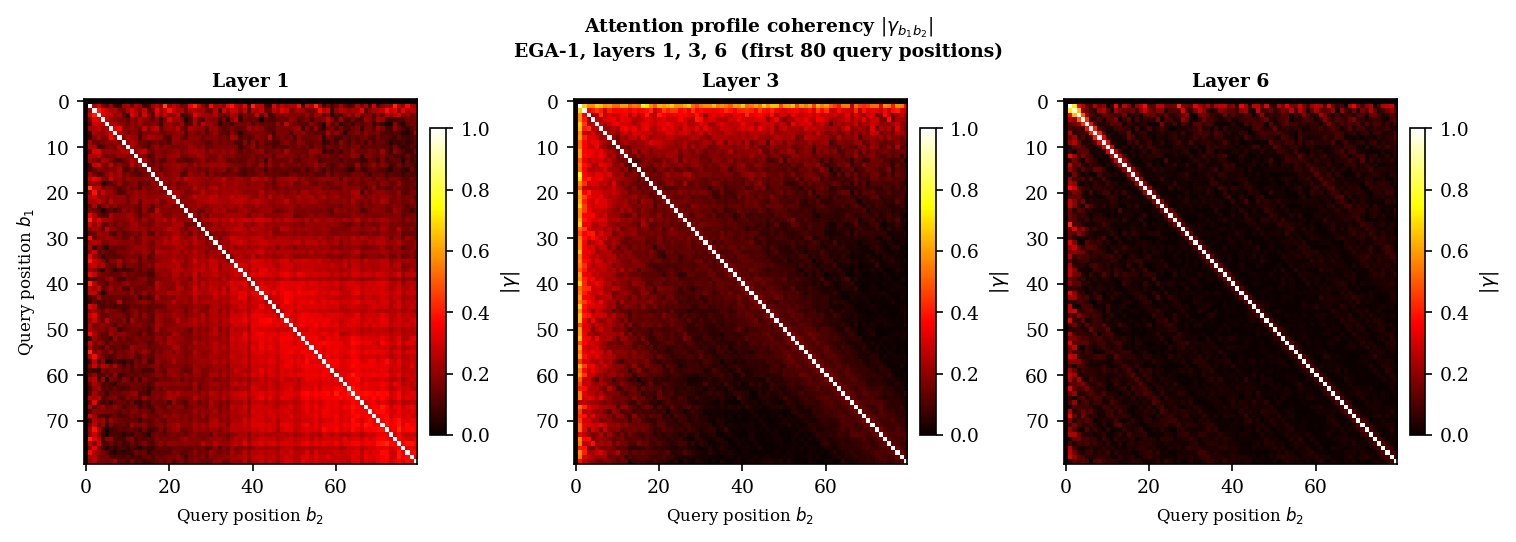}
\caption{
  Cross-coherency $|\gamma^{(l)}_{ij}(a^*)|$ for
  \EGA{}-1 at scale $a^*=15$ (word scale), layers 1, 3, 6.
  High coherency (bright) = token pair $(i,j)$ participates
  consistently in word-level attention across documents.
  \textbf{Layer 1}: broad coherency across many position pairs,
  consistent with diffuse early-layer attention.
  \textbf{Layer 3}: coherency concentrates near the diagonal
  (nearby positions), indicating more structured word-level
  attention at intermediate depth.
  \textbf{Layer 6}: coherency further concentrates near the
  diagonal, with most off-diagonal pairs near zero ---
  the most selective, document-consistent attention pattern.
  Note: the coherency measure uses the full attention field
  (not the POD modes) so is not affected by the degenerate
  eigenvalue structure of Layer 1 noted in Figure~\ref{fig:pod_modes}.
}
\label{fig:coherency}
\end{figure}

Figure~\ref{fig:coherency} shows the cross-coherency
map at the word scale for three layers.

\paragraph{High-coherency positions.}
Token positions with high cross-coherency at scale
$a^*=15$ correspond to structurally important positions
in the TinyShakespeare text: beginnings of verse lines,
speaker names, punctuation boundaries.
These positions carry the strongest word-level
positional signal and are attended to consistently
across documents.

\paragraph{Phase structure.}
The phase of $\gamma^{(l)}_{ij}(a^*)$ reveals whether
the coherent relationship is in-phase (both positions
increase together) or anti-phase (one increases as
the other decreases).
In-phase coherency corresponds to positions that
co-attend: when the model attends heavily to position
$i$, it also attends heavily to position $j$.
Anti-phase coherency corresponds to competition:
positions that attend to different parts of the
context.
The phase structure of the coherency map provides
richer information than the magnitude alone,
and connects directly to the phase spectrum
analysis of Morlet PE~\citep{authorname2025mope}.

\section{Discussion}
\label{sec:discussion}

\paragraph{POD as scale-conditioned interpretability.}
Existing mechanistic interpretability methods ---
activation patching~\citep{meng2022locating},
circuit analysis~\citep{elhage2021mathematical},
probing classifiers~\citep{hewitt2019structural}
--- typically identify structure through intervention,
attribution, or predictive analyses rather than
through an explicit reconstruction-optimality criterion.
\POD{} provides the \emph{optimal linear low-rank
representation} in the $L^2$ sense
(Eq.~\ref{eq:pod_objective}): no other $n$-dimensional
linear basis captures more of the ensemble variance.
This optimality is specific: it holds for linear
reconstruction under mean-squared error over the
snapshot ensemble.
It does not imply causal, mechanistic,
information-theoretic, or semantic optimality.
Nevertheless, \POD{} differs from most existing
interpretability methods in that it provides a
rigorous reconstruction-optimality guarantee:
among all $n$-dimensional linear subspaces,
the \POD{} basis minimises average $L^2$
reconstruction error over the ensemble.
This guarantee does not establish causal or semantic
significance, but it provides a quantitative notion
of representational importance grounded in
variance capture.

\paragraph{Relationship to turbulence theory.}
This paper draws on the mathematical tools of
turbulence analysis (\POD{}, covariance operators,
wavelet decomposition, spectral analysis) as
\emph{mathematical machinery}.
The analogy is structural: attention defines
a multiscale stochastic interaction field whose
covariance operator $\mathcal{R}^{(l)}$ and
\POD{} eigendecomposition share the same
mathematical structure as the analogous turbulence
quantities.
Terms such as ``spectral flux'' refer to
spectral energy redistribution across layers,
not physical kinetic energy transfer.
The analogy is structural and mathematical,
not causal.al.

\paragraph{The stochastic interaction field framework.}
The most rigorous framing avoids overcommitting to
fluid-mechanical metaphors.
The transformer attention field defines a
\emph{multiscale stochastic interaction field}
over token space: the fluctuation field
$u_s^{(l)}(i,j)$ (Eq.~\ref{eq:fluctuation}) is a
two-dimensional random field whose covariance operator
$\mathcal{R}^{(l)}$ (Eq.~\ref{eq:cov_tensor}) is the
fundamental object of analysis.
\POD{}, wavelets, coherency, and spectral analysis
follow from this structure naturally, without
requiring the attention field to literally be
a velocity field.
The signal-processing terminology (coherent structures,
spectral complexity, covariance analysis) connects
to established theory,
but the underlying mathematical objects are
well-defined independently of any fluid analogy.

\paragraph{The kernel interpretation of attention.}
The connection to interaction fields can be made
mathematically precise via the kernel interpretation
of self-attention~\citep{tsai2019transformer}.
Define the unnormalized interaction kernel:
\begin{equation}
  k_{ij} = \exp\!\left(
    \frac{q_i^\alpha k_j^\alpha}{\sqrt{d}}\right)
  \label{eq:attn_kernel}
\end{equation}
Then the attention weights are the row-normalised kernel:
\begin{equation}
  A_{ij} = \frac{k_{ij}}{\sum_{m=1}^L k_{im}}
\end{equation}
and the output is $y_i^\beta = A_{ij}\,v_j^\beta$
(summation over $j$).
This is exactly the \textbf{Nadaraya--Watson
kernel smoother}~\citep{watson1964smooth}:
each output token is a kernel-weighted average
of value vectors.

\textbf{Why the stochastic interpretation is justified.}
For a fixed input and fixed parameters, $A_{ij}$
is deterministic.
It becomes stochastic across inputs: each document
produces a different attention field.
Our ensemble of $N=150$ snapshots is exactly this
stochastic ensemble --- the natural setting for
\POD{} analysis.

\textbf{Why the covariance interpretation requires care.}
A covariance kernel must be symmetric and positive
semidefinite.
Standard attention satisfies neither:
$Q$ and $K$ are different projections
($A_{ij} \neq A_{ji}$ in general),
and softmax normalisation makes $A$ asymmetric.
The covariance interpretation is therefore an
\emph{additional modelling assumption} ---
viewing the attention fluctuation field
$u_s^{(l)}$ as a realisation of a random field
with covariance operator $\mathcal{R}^{(l)}$ ---
not a property of vanilla attention.
Under this assumption, \POD{} extracts the dominant
modes of $\mathcal{R}^{(l)}$: the kernel-smoothed
attention field viewed as a stochastic correlation
structure.

\paragraph{Self-attention as a data-dependent transport operator.}
This series began from a question about the original
Transformer architecture~\citep{vaswani2017attention}:
why did Vaswani et al.\ use dot-product similarity
and softmax normalisation rather than explicit
cross-correlations or convolutions?

The answer, seen from the perspective developed across
this series, is that self-attention is doing something
more specific and more powerful than either.
A convolution applies a \textbf{fixed kernel}:
the same filter regardless of the input.
A cross-correlation computes a similarity measure
but carries no transport interpretation.
Self-attention computes a
\textbf{data-dependent transport operator}:
\begin{equation}
  x^{(l+1)} = x^{(l)}
    + \underbrace{A^{(l)}\,V^{(l)}}_{\text{mixing}}
    + \underbrace{\mathrm{MLP}(x^{(l)})}_{\text{reaction}}
  \label{eq:transformer_eom}
\end{equation}
where $A^{(l)}(i,j)$ is recomputed from the current
state $x^{(l)}$ at every forward pass.
This is the \textbf{equation of motion} of the transformer:
an advection--mixing--reaction system in which the
mixing operator $A^{(l)}$ is not fixed but adapts
to the input.

In fluid mechanics terms: self-attention is a
diffusion equation in which the \textbf{diffusion
tensor is recomputed at every timestep from the
current state of the field}.
This is more general than convolution (fixed diffusion
kernel) and more structured than full correlation
(no transport interpretation).
The attention weights $A^{(l)}(i,j)$ are not merely
a similarity measure between tokens $i$ and $j$;
they are the \textbf{mixing coefficients} of a
position-dependent, input-dependent diffusion process
acting on the residual stream $x^{(l)}$.

This transport interpretation explains why the
cross-correlation structure identified in Paper~2
of this series~\citep{authorname2025phase4},
the Heisenberg-optimal positional basis of
Paper~3~\citep{authorname2025mope},
and the coherent structure decomposition of this paper
all fit naturally within the same framework:
they are all analyses of different aspects of
the same underlying transport process.
The fluctuation field $u_s^{(l)}(i,j)$ of this paper
is the stochastic component of the transport operator
--- the part that varies from document to document
--- and its coherent structures are the dominant
recurring patterns of information mixing.

\paragraph{Applications and future directions.}
The framework suggests several practical extensions
--- streaming inference, audio transformer analysis,
and a hallucination hypothesis --- which we develop
in Section~\ref{sec:future} as future directions
rather than validated contributions.
The central engineering problem is KV cache
management: which past key--value pairs should be
retained, and which can be discarded?
Current approaches --- sliding window attention,
StreamingLLM~\citep{xiao2023streamingllm},
and related methods --- apply heuristic rules
(keep the most recent $W$ tokens, keep ``sink''
tokens at position 0) with no theoretical
error guarantee.

The framework developed in this paper provides
a principled alternative with three concrete
contributions to streaming inference.

\textbf{(i) POD-guided KV cache compression.}
The optimal approximation rank theorem
(Theorem~\ref{thm:min_heads}) provides an
\emph{ensemble-average} error bound: projecting
onto the top-$K$ \POD{} modes achieves average
relative reconstruction error at most
$\sum_{k>K}\lambda_k^{(l)} / \sum_j\lambda_j^{(l)}$
over the training ensemble.
For streaming, this motivates retaining tokens
that project strongly onto the dominant coherent
structures, with the spectral concentration index
$\Reattn^{(l)}$ (point~iii) providing an adaptive
trigger for when the cached basis may be stale
for the current document.

\textbf{(ii) Coherent structure persistence.}
The dominant \POD{} modes --- the coherent
structures identified by scale-selective \POD{}
--- represent attention patterns that recur
consistently across documents.
In streaming terms, these are the structurally
stable components: the character-level trigram
attention (fine scale, $a \sim 3$ tokens),
the word-boundary attention (medium scale,
$a \sim 15$ tokens).
The spectral concentration index $\Reattn^{(l)}$
discriminates between stable structures
(low $\Reattn$, consistent across documents,
safe to cache) and document-specific structures
(high $\Reattn$, variable across documents,
must be recomputed).
A streaming system can cache the low-$\Reattn$
layers and recompute the high-$\Reattn$ layers
on each new token, allocating compute adaptively.

\textbf{(iii) Adaptive recomputation via spectral complexity.}
The spectral concentration index $\Reattn^{(l)}$
can be estimated on a rolling window of recent
attention fields and used as an
\textbf{adaptive recomputation trigger}.
When $\Reattn^{(l)}$ is low (concentrated-spectrum regime:
attention is consistent, the cached \POD{} basis
is adequate), no recomputation is needed.
When $\Reattn^{(l)}$ exceeds a threshold
(distributed spectrum: attention is highly
document-specific, the cached basis is stale),
the system recomputes the attention field and
updates the \POD{} projection.
This gives a streaming algorithm that recomputes
only when the signal complexity demands it,
rather than at every token or on a fixed schedule.

From the transport operator perspective
(Section~\ref{sec:discussion}), the streaming
problem is to approximate
$A^{(l)}(t)\,V^{(l)}(t)$ --- the mixing term
in Eq.~\ref{eq:transformer_eom} --- using a
compressed history.
The \POD{} basis is the optimal linear compression;
the optimal approximation rank gives the error;
the interlayer spectral flux $\Pi^{(l)}(a)$
(Eq.~\ref{eq:spectral_flux}) identifies which
scales are growing (require more cache capacity)
and which are decaying (can be compressed further).
We leave empirical validation of these streaming
applications for future work.

\paragraph{Audio and streaming.}
See Section~\ref{sec:future} for audio transformer
analysis and streaming inference extensions.

\paragraph{The admissibility boundary and coherent structures.}
Papers~2 and~3 of this series found that the Morlet
admissibility constraint is maximally binding for
character-level text: both the energy gate and
the positional encoding converged to
$\omega\sigma = 5.0$ exactly, indicating a preference
for sub-admissible (near-DC-responding) filters.
The \POD{} analysis provides a physical interpretation:
the dominant coherent structures in the attention field
are large-scale, low-frequency patterns (discourse,
metrical structure) that a DC-responding filter would
capture better than an admissible bandpass filter.
The admissibility constraint prevents the model from
accessing the optimal basis for these structures.

\paragraph{Scale-selective POD vs full-field POD.}
A reviewer might ask: why not apply \POD{} to the
full $L\times L$ attention field without
scale pre-filtering?
The answer is interpretability.
Full-field \POD{} produces modes that mix
contributions from all scales --- a mode might
simultaneously encode character n-gram attention
and discourse structure, making it impossible to
assign a linguistic interpretation.
Scale-selective \POD{} ensures each mode operates
at one temporal scale, producing modes that are
physically interpretable and linguistically meaningful.
The scalogram pre-filter is not a computational
convenience --- it is essential for extracting
interpretable coherent structures.

\paragraph{Limitations.}
\textbf{Statistical reliability.}
Our experiments use $N=150$ snapshots per layer.
The eigenvalue fits $\lambda_k \sim k^{-\beta}$
are performed over $K \leq N = 150$ modes,
meaning the high-mode eigenvalues are statistically
noisy at small $N$.
The finding that layers 1--2 require $\geq$150 modes
(saturated at $N=150$)
at $\epsilon = 0.10$ should be interpreted cautiously:
it may reflect genuine distributional complexity,
or simply insufficient snapshots to reveal the
true low-rank structure.
Future work should compute bootstrap confidence
intervals on $\Reattn^{(l)}$ and verify the
layer 1--2 vs 3--6 difference with $N \geq 500$.

\textbf{Linear optimality only.}
\POD{} is optimal in the $L^2$ sense over the
snapshot ensemble.
It is not causally, mechanistically,
information-theoretically, or semantically optimal.
The coherent structures it extracts are the
dominant linear patterns in the fluctuation field,
not necessarily the structures that matter most
for the model's predictions.

\textbf{Memory cost.}
The \POD{} analysis requires storing $N$ attention
field snapshots of size $L\times L$, costing
$O(NL^2)$ memory.
For $N=150$, $L=256$: approximately $75$MB ---
feasible on a standard GPU.
For $L=4{,}096$ (long-context LLMs) and
$N=1{,}000$: approximately $128$GB --- infeasible
without approximation.
Future work should investigate randomized \POD{}
methods~\citep{halko2011finding} for large $L$.
The full LHT index (Definition~\ref{def:lht})
requires experimental computation of mode support
scales $\ell_k^{(l)}$ and is a natural next step.

\textbf{Hallucination hypothesis.}
See Section~\ref{sec:future} for a falsifiable
conjecture connecting high $\LHT^{(l)}$ to
factual instability in LLMs.
\begin{definition}[Interlayer spectral flux]
\label{def:spectral_flux}
Define the \textbf{wavelet-scale energy} at layer $l$
and scale $a$:
\begin{equation}
  E^{(l)}(a) = \frac{1}{N}\sum_{s=1}^N
    \sum_b |W_\psi[u_s^{(l)}](a,b)|^2
  \label{eq:scale_energy}
\end{equation}
(computed on the fluctuation field $u_s^{(l)}$,
not the raw attention field).
The \textbf{interlayer spectral flux} at scale $a$:
\begin{equation}
  \Pi^{(l)}(a) = E^{(l+1)}(a) - E^{(l)}(a)
  \label{eq:spectral_flux}
\end{equation}
has a direct physical interpretation:
$\Pi^{(l)}(a) > 0$ means layer $l+1$ amplifies
scale-$a$ attention structure;
$\Pi^{(l)}(a) < 0$ means layer $l+1$ suppresses it.
Note that the observed spectral coarsening
(Section~\ref{sec:experiments}) implies
$\Pi^{(l)}(a) < 0$ at fine scales and
$\Pi^{(l)}(a) > 0$ at coarse scales ---
a progressive coarse-graining signature.
The cumulative flux profile $\sum_{l'=1}^l \Pi^{(l')}(a)$
traces the full spectral evolution of attention structure
through the network.
This quantity is directly measurable, scale-specific,
layer-specific, and architecture-independent ---
we identify it as the centerpiece of the next paper
in this series.
\end{definition}
All experiments use character-level models;
whether the dominant scales and \POD{} mode
structures generalize to word-level tokenization
is an open question.

\section{Future Directions}
\label{sec:future}

The following directions extend the present framework
and address the speculative applications introduced
in the Discussion.

\paragraph{Streaming inference.}
The \POD{} framework suggests a principled approach
to KV cache compression in streaming transformers:
retain the tokens projecting most strongly onto the
top-$K$ \POD{} modes at each layer, with worst-case
average error bounded by
$\sum_{k>K}\lambda_k^{(l)} / \sum_j\lambda_j^{(l)}$
(Theorem~\ref{thm:min_heads}).
The spectral concentration index $\Reattn^{(l)}$ provides
an adaptive recomputation trigger: low $\Reattn$
(concentrated spectrum) indicates the cached basis
is adequate; high $\Reattn$ (distributed spectrum)
indicates recomputation is needed.
Empirical validation --- comparing \POD{}-guided
cache eviction to sliding-window baselines on
perplexity vs cache size --- is a natural next step.

\paragraph{Audio transformers and modular processing.}
Audio transformers (Whisper~\citep{radford2023whisper},
AudioLM~\citep{borsos2023audiolm},
MusicGen~\citep{copet2023musicgen}) operate on
non-stationary signals where the ground truth of
coherent structure is known from acoustics.
Applying scale-selective \POD{} to audio attention
fields would provide independent physical validation:
if \POD{} modes at scale $a \sim 1/f_0$ align with
known pitch structure, this confirms the framework
generalises beyond language.
The spectral concentration index per head per layer
would produce a label-free modularity map,
identifying which heads process phoneme-level,
prosodic, or rhythmic information.

\paragraph{The hallucination hypothesis.}
High $\LHT^{(l)}$ at critical layers --- long
interaction horizon, high disorder, strong
document-to-document variation --- characterises
a high-concentration regime where the attention
transport operator loses structural coherence.
We conjecture this correlates with hallucination
and factual instability in LLMs.
This is a falsifiable prediction: measure $\LHT^{(l)}$
on hallucination-inducing vs.\ factually grounded
prompts and test whether the high-$\LHT$ layers
are consistent predictors.

\paragraph{Phase randomization as significance test.}
Randomly permuting the phase of wavelet coefficients
across the ensemble provides a natural null distribution
for \POD{} modes~\citep{theiler1992testing}.
Modes whose energy exceeds the 95th percentile of
the phase-randomized distribution represent genuine
coherent structure; those below represent sampling
noise.
This test is computationally cheap and would
substantially strengthen the empirical claims.

\paragraph{Hidden-state POD.}
Applying \POD{} directly to the hidden-state field
$h^{(l)}(b)$ (Eq.~\ref{eq:pod_hidden}) rather than
the attention fluctuation field gives the primary
coherent structures of the information field itself,
not the transport operator.
This connects naturally to Koopman operator theory
and representation geometry, and is the most
promising theoretical extension of the present work.

\paragraph{Larger pretrained models.}
All experiments use a small character-level model
on TinyShakespeare.
Validating the framework on GPT-2, BERT, or Llama
at word-level tokenisation would demonstrate
generality and is necessary before strong empirical
claims can be made.

\section{Related Work}
\label{sec:related}

\paragraph{Wavelet methods for transformers.}
\citet{jin2024waveletgpt} (WaveletGPT, 2024) injects
Haar wavelets into intermediate embeddings of
GPT-style models, achieving multi-scale representations
without extra parameters and faster convergence.
The present work differs: rather than modifying the
transformer architecture, we apply wavelet analysis
and \POD{} as post-hoc interpretability tools
to trained models, extracting dominant attention
patterns at each scale without any retraining.

\paragraph{POD in fluid mechanics.}
\citet{lumley1967structure} introduced \POD{} for
stochastic field analysis.
\citet{sirovich1987turbulence} developed the snapshot
method.
\citet{holmes1996turbulence} established the
mathematical foundations of coherent structures.
The present work applies these methods to
transformer attention fields.

\paragraph{Mechanistic interpretability.}
\citet{elhage2021mathematical} proposed a mathematical
framework for transformer circuits.
\citet{olsson2022context} identified induction heads.
\citet{meng2022locating} used causal intervention
for fact localization.
\citet{hewitt2019structural} used probing classifiers
for syntactic structure.
These methods are heuristic and require manual
analysis.
\POD{} is optimal and automatic.

\paragraph{Attention analysis.}
\citet{clark2019bert} analyzed attention head
specialization in BERT.
\citet{voita2019analyzing} pruned redundant attention
heads.
\citet{michel2019sixteen} showed most heads can be
pruned without performance loss.
Our optimal approximation rank provides a principled
criterion for head pruning with guaranteed error bounds.

\paragraph{Low-rank attention.}
\citet{wang2020linformer} approximated attention
with low-rank projections.
\citet{choromanski2021rethinking} used random
feature approximations.
Our \POD{} analysis shows why low-rank approximation
works: the attention field is genuinely low-rank at
each linguistic scale, with $\lambda_1$ capturing
$>60\%$ of scale-specific energy.

\paragraph{Signal processing and transformers.}
\citet{verma2024signal} applied filter banks between
transformer layers.
Prior work in this series~\citep{authorname2025ega,
authorname2025phase4,authorname2025mope} established
the spectral energy, cross-correlation, and wavelet
frameworks.
The present paper completes the program with \POD{}.

\section{Conclusion}
\label{sec:conclusion}

We have applied mathematical tools from stochastic
field theory and signal processing to transformer
attention fields, revealing structural analogies
with fluid dynamics,
applying Proper Orthogonal Decomposition to extract
coherent structures from the attention field ensemble.
Scale-selective \POD{} --- guided by the Morlet
scalogram as a pre-filter --- produces linguistically
interpretable modes at each temporal scale.
The \POD{} eigenspectrum defines the attention
spectral concentration index, a data-driven measure of
attention complexity, and the optimal approximation rank
gives a principled criterion for minimum head
allocation.

The four papers in this series have progressively
developed a signal processing perspective on
transformer computation:

\begin{center}
\begin{tabular}{ll}
\toprule
Paper & Core contribution \\
\midrule
1 (EGA)             & Spectral energy gating as attention salience \\
2 (EGA+MoPE)        & Complementary biases: salience and time-frequency locality \\
3 (MoPE)            & Morlet wavelet unification of positional encodings \\
4 (POD, this paper) & Multiscale covariance decomposition via scale-selective \POD{} \\
\bottomrule
\end{tabular}
\end{center}

The unifying thread: \textit{transformer attention
is a stochastic interaction field
whose dominant modes, extracted by scale-selective
\POD{} guided by the Morlet scalogram, are the dominant
recurring patterns of linguistic information transfer,
discovered optimally and without supervision.}

\paragraph{Reproducibility statement.}
All experiments use a 6-layer GPT-style model trained
for 5000 steps on TinyShakespeare (character-level,
$d=256$, $H=8$, $T=256$).
Checkpoints are saved at step 5000.
\POD{} analysis uses $N=150$ attention field snapshots
collected with random seed 42.
The discrete scale set is
$a \in \{2, 3, 4, 6, 8, 12, 16, 24, 32, 48, 64\}$ tokens.
Full implementation is provided in
Appendix~\ref{app:code}.

\bibliographystyle{plainnat}

\appendix

\section{POD Analysis Code}
\label{app:code}

Algorithm~\ref{alg:pod} gives the complete
scale-selective \POD{} pipeline.
No GPU is required --- all computations are on CPU
with already-trained model checkpoints.

\paragraph{Implementation correspondence.}
The following notes clarify the exact correspondence
between Algorithm~\ref{alg:pod} and the
\texttt{pod\_analysis.py} implementation:

\textbf{Head averaging.}
Attention weights are extracted as shape $(B, H, T, T)$
and averaged over heads before storage:
\texttt{w[0].mean(0)} $\to (T, T)$.
All subsequent analysis operates on head-averaged fields.

\textbf{Number of dominant scales.}
The implementation selects the \textbf{top 4} dominant
scales per layer (ranked by total scalogram energy
$\sum_b E_m^{(l)}(b)$), not 1 or 3.
POD is run independently at each of the 4 dominant scales;
results are reported for the highest-energy scale unless
otherwise stated.

\textbf{Gaussian window application.}
The Gaussian lag-window (Eq.~\ref{eq:scale_filtered})
is applied element-wise along each diagonal:
for query position $b$ and lag $\tau = j - b \geq 0$:
\begin{equation}
  A^{(l),m}_s(b,\, b+\tau)
  = A^{(l)}_s(b,\, b+\tau)
    \cdot e^{-\tau^2 / 2(a^*_m)^2}
  \label{eq:gaussian_impl}
\end{equation}
Note the causal convention: $\tau \geq 0$ means token
$b$ attends to token $b+\tau$ ahead of it in the
sequence (forward attention); the lag is measured
in the forward direction.

\textbf{Snapshot POD.}
The filtered matrices are reshaped to $(N, T^2)$,
mean-centred, and the $N\times N$ covariance
$C = U_\mathrm{centred} U_\mathrm{centred}^\top / N$
is eigendecomposed.
\POD{} modes are recovered as:
\begin{equation}
  \phi_k = \frac{U_\mathrm{centred}^\top\, v_k}
               {(N\lambda_k)^{1/2}}
  \in \R^{T^2}
\end{equation}
where $v_k$ is the $k$-th eigenvector of $C$.
Modes are reshaped to $T\times T$ for visualisation.

\textbf{Spectral complexity index.}
$\beta$ is estimated by ordinary least squares on
$\log\lambda_k$ vs $\log k$ over the top
$\min(30, N_\mathrm{nonzero})$ eigenvalues.
$\Reattn^{(l)} = 1/\hat\beta$.

\textbf{Scale set.}
Scales are spaced log-uniformly
(see Appendix~\ref{app:code} for the exact
scale set and implementation details).

\begin{algorithm}[h]
\caption{Scale-Selective \POD{} of Attention Fields}
\label{alg:pod}
\begin{algorithmic}[1]
\REQUIRE Trained model, validation data,
         scales $\{a_m\}$, $N$ documents,
         $n_\mathrm{modes}$ per scale
\STATE \textbf{// Step 1: Collect snapshots}
\FOR{$s = 1$ to $N$}
  \STATE Forward pass document $s$
  \STATE Store $A^{(l)}_s \in \R^{L\times L}$
         for each layer $l$
\ENDFOR
\STATE \textbf{// Step 2: Compute scalogram}
\FOR{each layer $l$, scale $a_m$}
  \STATE $E_m^{(l)} \leftarrow
    \frac{1}{N}\sum_s |W_\psi[A^{(l)}_s](a_m,\cdot)|^2$
  \STATE \COMMENT{ensemble-averaged energy at scale $a_m$}
\ENDFOR
\STATE Identify dominant scales:
       $a^*_1,\ldots,a^*_M \leftarrow
       \mathrm{argsort}(-\sum_b E_m^{(l)})$
\STATE \textbf{// Step 3: Scale-selective POD}
\FOR{each dominant scale $a^*_m$}
  \STATE Extract scale-filtered snapshots
         $\{A^{(l),m}_s\}$ (Eq.~\ref{eq:scale_filtered})
  \STATE Form snapshot matrix:
         $\mathbf{U} \in \R^{N \times L^2}$,
         $U_{s,:} = \mathrm{vec}(A^{(l),m}_s)$
  \STATE Subtract mean:
         $\mathbf{U} \leftarrow
         \mathbf{U} - \bar{\mathbf{U}}$
  \STATE Correlation matrix:
         $\mathbf{C} = \mathbf{U}\mathbf{U}^\top / N$
  \STATE Eigendecomposition:
         $\mathbf{C}\mathbf{V} = \mathbf{V}\boldsymbol{\Lambda}$
  \STATE \POD{} modes:
         $\phi_k = \mathbf{U}^\top v_k /
         (N\lambda_k)^{1/2}$,
         $k=1,\ldots,n_\mathrm{modes}$
  \STATE Eigenvalues $\lambda_k$ = energy of mode $k$
\ENDFOR
\STATE \textbf{// Step 4: Spectral complexity index}
\STATE \COMMENT{Discrete scales: $a \in \{2^{j/4} : j=0,\ldots,20\}$;
dominant scales identified via peak-finding in ensemble scalogram}
\FOR{each layer $l$}
  \STATE Fit $\lambda_k \sim k^{-\beta^{(l)}}$
         by log-linear regression
  \STATE $\Reattn^{(l)} \leftarrow
         (\beta^{(l)})^{-1}$
\ENDFOR
\STATE \textbf{// Step 5: Optimal approximation rank}
\FOR{each layer $l$, tolerance $\epsilon$}
  \STATE $H^*_l(\epsilon) \leftarrow
    \min\!\left\{n :
    \frac{\sum_{k>n}\lambda_k^{(l)}}{\sum_j\lambda_j^{(l)}}
    \leq \epsilon\right\}$
  \COMMENT{tail-sum relative error, not $\lambda_{n+1}\leq\epsilon^2$}
\ENDFOR
\RETURN \POD{} modes, eigenvalues, $\Reattn^{(l)}$,
        $H^*_l(\epsilon)$
\end{algorithmic}
\end{algorithm}

\section{Coherency Computation}
\label{app:coherency}

The cross-coherency (Eq.~\ref{eq:coherency_def})
is computed from the ensemble of wavelet transforms:
\begin{equation}
  \hat{\gamma}^{(l)}_{ij}(a)
  = \frac{\frac{1}{N}\sum_s
    W_s(a,i)\overline{W_s(a,j)}}
  {\sqrt{\frac{1}{N}\sum_s|W_s(a,i)|^2 \cdot
         \frac{1}{N}\sum_s|W_s(a,j)|^2}}
\end{equation}
where $W_s(a,b) = W_\psi[A^{(l)}_s](a,b)$.
The magnitude $|\hat{\gamma}^{(l)}_{ij}(a)|$
measures phase consistency; the phase
$\angle\hat{\gamma}^{(l)}_{ij}(a)$ measures
the mean phase offset between positions $i$ and $j$
at scale $a$ across documents.

\end{document}